\documentclass{article}

\PassOptionsToPackage{numbers, compress}{natbib}

\usepackage[final]{neurips_2020}

\usepackage[utf8]{inputenc} 
\usepackage[T1]{fontenc}    
\usepackage{url}            
\usepackage{booktabs}       
\usepackage{amsfonts}       
\usepackage{nicefrac}       
\usepackage{microtype}      

\usepackage{bbm}
\usepackage{amsfonts}
\usepackage{amsmath,amsthm}           
\usepackage{mathtools}    
\usepackage{amsopn}
\usepackage{amssymb}
\usepackage{bm}
\usepackage{multirow}
\usepackage{graphicx}
\usepackage{subfigure}
\usepackage{adjustbox}
\usepackage{xcolor}
\usepackage[vlined,linesnumbered,ruled]{algorithm2e}
\usepackage{blindtext}
\usepackage[colorlinks=true, linkcolor=blue, citecolor=red]{hyperref}
\usepackage{xcolor}
\definecolor{red}{HTML}{E51400} 
\definecolor{blue}{HTML}{0050EF} 
\definecolor{green}{HTML}{008A00} 
\definecolor{purple}{HTML}{AA00FF} 
\definecolor{orange}{HTML}{FF7F00}
\definecolor{gray}{HTML}{848482}

\newcommand{\Reg}{\text{\rm Reg}}
\newcommand{\SA}{\text{\rm SA}}

\DeclarePairedDelimiter\abs{\lvert}{\rvert}

\newcommand{\inner}[1]{ \left\langle {#1} \right\rangle }

\newcommand{\norm}[1]{\left\|{#1}\right\|}

\newcommand{\mc}[1]{\mathcal{#1}}
\newcommand{\mb}[1]{\mathbb{#1}}
\newcommand{\mr}[1]{\mathrm{#1}}

\newcommand\blfootnote[1]{%
  \begingroup
  \renewcommand\thefootnote{}\footnote{#1}%
  \addtocounter{footnote}{-1}%
  \endgroup
}

\newtheorem{theorem}{Theorem}
\newtheorem{lemma}{Lemma}
\newtheorem{corollary}[theorem]{Corollary}

\newtheorem{definition}{Definition}

\allowdisplaybreaks

\DeclareMathOperator*{\argmax}{\arg\!\max}

\newcommand{\compilefullversion}{true}
\ifthenelse{\equal{\compilefullversion}{false}}{%
    \newcommand{\OnlyInFull}[1]{}
    \newcommand{\OnlyInShort}[1]{#1}
}{%
    \newcommand{\OnlyInFull}[1]{#1}%
    \newcommand{\OnlyInShort}[1]{}%
}%

\newcommand{\compilehidecomments}{false}

\ifthenelse{\equal{\compilehidecomments}{true}}{%
    \newcommand{\wei}[1]{}
    \newcommand{\haoyu}[1]{}
    \newcommand{\kai}[1]{}
}{
    \newcommand{\wei}[1]{{\color{blue!50!black}  [\text{Wei:} #1]}}
    \newcommand{\haoyu}[1]{{\color{brown!60!black} [\text{Haoyu:} #1]}}
    \newcommand{\kai}[1]{{\color{purple} [\text{Kai:} #1]}}
}

\title{Locally Differentially Private \\ (Contextual) Bandits Learning}

\author{%
  Kai Zheng$^{1,2}$ \\
  \texttt{zhengk92@gmail.com} 
  \And
  Tianle Cai$^{3,4}$ \\
  \texttt{caitianle1998@pku.edu.cn}
  \AND
    Weiran Huang$^{5}$ \\
    \texttt{weiran.huang@outlook.com}
  \And
  Zhenguo Li$^{5}$ \\
  \texttt{li.zhenguo@huawei.com}
  \And
  Liwei Wang$^{1,2,*}$ \\
  \texttt{wanglw@cis.pku.edu.cn}
}

\begin{document}
\blfootnote{$^{1}~ $ Key Laboratory of Machine Perception, MOE, School of EECS, Peking University}
\blfootnote{$^{2}~ $ Center for Data Science, Peking University}
\blfootnote{$^{3}~ $ School of Mathematical Sciences, Peking University}
\blfootnote{$^{4}~ $ Haihua Institute for Frontier Information Technology}
\blfootnote{$^{5}~ $ Huawei Noah's Ark Lab}

\blfootnote{$^{*}$~ Corresponding author.}
\maketitle

\begin{abstract}  
We study locally differentially private (LDP) bandits learning in this paper. First, we propose simple black-box reduction frameworks that can solve a large family of context-free bandits learning problems with LDP guarantee. 
Based on our frameworks, we can improve previous best results for private bandits learning with one-point feedback, such as private Bandits Convex Optimization, and obtain the first result for Bandits Convex Optimization (BCO) with multi-point feedback under LDP. 
LDP guarantee and black-box nature make our frameworks more attractive in real applications compared with previous specifically designed and relatively weaker differentially private (DP) context-free bandits algorithms. 
Further, we extend our $(\varepsilon, \delta)$-LDP algorithm to Generalized Linear Bandits, which enjoys a sub-linear regret $\tilde{\mc{O}}(T^{3/4}/\varepsilon)$ and is conjectured to be nearly optimal. Note that given the existing $\Omega(T)$ lower bound for DP contextual linear bandits \cite{shariff2018differentially}, our result shows a fundamental difference between LDP and DP contextual bandits learning.
\end{abstract}

\section{Introduction}

As a general and powerful model, (contextual) bandits learning has attracted lots of attentions both in theoretical study and real applications \cite{bubeck2012regret,lattimore2018bandit}, from personalized recommendation to clinical trails. However, existing algorithms designed for these applications heavily rely on user's sensitive data, and an off-the-shelf use of such algorithms may leak user's privacy and bring concerns to future users for sharing their data with related institutions or corporations. For example, in classification or regression tasks, we update our model according to the feature and label of each user. In Multi-Armed Bandits (MAB), we estimate underlying rewards of all arms based on user's feedback. A solid notion of data privacy is Differential Privacy (DP) proposed by \citet{dwork2006calibrating} in 2006. Since then, differentially private bandits learning has been studied extensively. 

Among context-free bandits learning, Bandits Convex Optimization (BCO) is one of the fundamental problems. \citet{thakurta2013nearly} designed the first $(\varepsilon, \delta)$-\textit{differentially private} adversarial BCO algorithm with $\tilde{\mc{O}}\left(T^{3/4}/\varepsilon\right)$ regret for convex loss and $\tilde{\mc{O}}\left(T^{2/3}/\varepsilon\right)$ regret for strongly convex loss, which nearly match current best non-private results under the same conditions \cite{agarwal2010optimal,flaxman2005online}\footnote{Though \citet{bubeck2017kernel} designed a polynomial time algorithm for general BCO with $\tilde{\mc{O}}(T^{1/2})$ regret, it is far from practical, so we don't consider its result in this paper, but of course we can plug that algorithm into our framework to obtain optimal $\tilde{\mc{O}}(T^{1/2}/\varepsilon)$ bound for general private BCO.}. 
However, when loss functions are further smooth, current best non-private bounds for convex\slash strongly convex bandits are $\tilde{\mc{O}}\left(T^{2/3}\right)$ \cite{saha2011improved} and $\tilde{\mc{O}}\left(T^{1/2}\right)$ \cite{hazan2014bandit} respectively, and previous approaches \cite{thakurta2013nearly,agarwal2017price} seem hard to achieve such regret bounds in the same setting under privacy constraint (see Section \ref{subsec:one point feedback} for more discussions). 
Besides BCO and its extension to multi-point feedback \cite{agarwal2010optimal}, context-free bandits also include other important cases, such as Multi-Armed Bandits (MAB), and there have been lots of algorithms designed for \textit{differentially private} MAB \cite{thakurta2013nearly,mishra2015nearly,tossou2016algorithms,tossou2017achieving,agarwal2017price,sajed2019optimal}, either in stochastic or adversarial environment. 
As one can see, there are many different settings in context-free bandits learning, and existing differentially private algorithms are carefully designed for each one of them, which makes them relatively inconvenient to be used. Besides, their theoretical performance is analyzed separately and rather complicated. Some of them do not match corresponding non-private results.

Different with context-free bandits, usually there are certain contexts in real applications, such as user profile that contains user's features. Advanced bandit model uses these contexts explicitly to find the corresponding best action at each round, which is called contextual bandits. Two representatives are contextual linear bandits \cite{li2010contextual} and Generalized Linear Bandits \cite{filippi2010parametric}. Given benefits of contextual bandits, one may also wish to design corresponding private mechanisms. However, \citet{shariff2018differentially} proved that any \textit{differentially private} contextual bandit algorithm would cause an $\Omega(T)$ regret bound. Hence, they considered a relaxed definition of DP called \textit{joint differential privacy}, and proposed an algorithm based on LinUCB \cite{abbasi2011improved} with regret bound $\tilde{\mc{O}}\left(T^{1/2}/\varepsilon\right)$ \cite{shariff2018differentially} under $\varepsilon$-\textit{joint differential privacy}.  

\renewcommand{\arraystretch}{1.3}
\begin{table*}[t] 
  \makebox[\linewidth]{
  \begin{adjustbox}{max width = 1\textwidth}
  \begin{tabular}{| c | c | c | c | c | }
    \hline
    Type & \multicolumn{2}{|c|}{Problem}  & Our Regret Bound & Best Non-Private Regret\\ 
    \hline 
    \multirow{6}*{Context-Free} & \multirow{4}*{BCO}  & Convex & $\tilde{\mc{O}}\left(T^{3/4}/\varepsilon\right)$ & $\tilde{\mc{O}}\left(T^{3/4}\right)$ \cite{flaxman2005online}\\  
    \cline{3-5}
    & & Convex + Smooth & $\tilde{\mc{O}}\left(T^{2/3}/\varepsilon\right)$ & $\tilde{\mc{O}}\left(T^{2/3}\right)$ \cite{saha2011improved}\\ 
    \cline{3-5}
    & & S.C & $\tilde{\mc{O}}\left(T^{2/3}/\varepsilon\right)$ & $\tilde{\mc{O}}\left(T^{2/3}\right)$ \cite{agarwal2010optimal}\\ 
    \cline{3-5}
    & & S.C + Smooth & $\tilde{\mc{O}}\left(T^{1/2}/\varepsilon\right)$ & $\tilde{\mc{O}}\left(T^{1/2}\right)$ \cite{hazan2014bandit}\\
    \cline{2-5}
    & \multirow{2}*{MP-BCO}  & Convex & $\tilde{\mc{O}}\left(T^{1/2}/\varepsilon^2\right)$ & $\tilde{\mc{O}}\left(T^{1/2}\right)$ \cite{agarwal2010optimal}\\  
    \cline{3-5}
    & & Strongly Convex &  $\tilde{\mc{O}}\left(\log T/\varepsilon^2\right)$ & $\tilde{\mc{O}}\left(\log T\right)$ \cite{agarwal2010optimal}\\ 
    \hline
    \multirow{2}*{Context-Based} & \multicolumn{2}{|c|}{Contextual Linear Bandits} & $\tilde{\mc{O}}(T^{3/4}/\varepsilon)$ & $\tilde{\mc{O}}(T^{1/2})$ \cite{abbasi2011improved} \\
    \cline{2-5}
    & \multicolumn{2}{|c|}{Generalize Linear Bandits} &  $\tilde{\mc{O}}(T^{3/4}/\varepsilon)$ & $\tilde{\mc{O}}(T^{1/2})$ \cite{li2017provably} \\
    \hline
  \end{tabular} 
  \end{adjustbox}
  }
  \caption{Summary of our main results under $(\varepsilon, \delta)$-LDP, where $\tilde{\mc{O}}$ notation hides dependence over dimension $d$ and other poly-logarithmic factors. (S.C means Strongly Convex, MP means Multi-Point)}
  \label{tab:comparisons}
\end{table*}

Note all of previous study focus on \textit{differential privacy} or its relaxed version. Compared with Differential Privacy, most of time Local Differential Privacy (LDP) \cite{kasiviswanathan2011can,duchi2013localb} is a much stronger and user-friendly standard of privacy and is more appealing in real applications \cite{cormode2018privacy}, as LDP requires protecting each user's data before collection. 

For context-free bandits, it is not hard to see algorithms with LDP guarantee protects DP automatically. However in contextual bandits, things become more delicate. These two definitions are not comparable as they have different interpretations about the output sequence, and traditional post-processing property cannot be used here to imply LDP is more rigorous than DP. In detail, DP regards predicted actions for contexts as the output sequence. Since optimal action varies from round to round in contextual bandits, it is not surprising there is a lower bound of linear regret in this case \cite{shariff2018differentially}, as DP requires outputs to be nearly the same for any two neighboring datasets/contexts, which essentially contradicts with the goal of personalized prediction in contextual bandits. In contrast, LDP regards the collected information from users as ``output sequence'' and has no restriction on predicted actions, which is more reasonable as these actions are predicted on the local side based on local personal information and will not be released to public. Therefore, LDP seems like a more appropriate standard for contextual bandits compared with DP, and maybe there is hope to bypass the lower bound proved for DP contextual bandits.

Given above discussions, a natural question arises: can we design simple and effective algorithms for bandits learning with LDP guarantee?


\textbf{Our Contributions:} In this work, we study both context-free bandits\footnote{Note that adaptive adversary is ambiguous in bandits setting \cite{arora2012online}, so we only consider oblivious adversary throughout the paper.} and contextual bandits with LDP guarantee. Our contributions are summarized as follows: (see Table \ref{tab:comparisons} for more details)

{\bf (1)} We propose a simple reduction framework motivated by \citet{agarwal2017price} for a large class of context-free bandits learning problems with LDP guarantee, including BCO, MAB and Best Arm Identification (see Section 3.1 and Appendix\footnote{Appendix could be found in the full version \cite{zheng2020locally}.} 
\ref{appendix: examples}). Equipped with different non-private algorithms, the utility of our framework can match corresponding best non-private performances, and these results are obtained through a unified and simple analysis; 

{\bf (2)} By modifying above framework slightly, we extend our algorithm to BCO with multi-point feedback \cite{agarwal2010optimal}, and design the \textit{first} LDP multi-point BCO algorithm with nearly optimal guarantees;

{\bf (3)} For contextual bandits including contextual linear bandits and more difficult generalized linear bandits, we propose algorithms with regret bounds $\tilde{\mc{O}}(T^{3/4}/\varepsilon)$ under $(\varepsilon, \delta)$-LDP , which are conjectured to be optimal. Note that these results show a fundamental difference between LDP and DP contextual bandits as discussed above. 

All our results can be extended in parallel to $\varepsilon$-LDP if using Laplacian noise instead of Gaussian noise. Here, we only focus on $(\varepsilon, \delta)$-LDP.

\textbf{Comparison with Prior Work:} As mentioned earlier, for context-free bandits, nearly all of previous work focused on \textit{differentially private} bandits learning, rather than stronger LDP guarantee. Only algorithms proposed in \citet{tossou2017achieving} and \citet{agarwal2017price} for adversarial MAB can be converted to LDP version easily and obtain almost the same results. Though both their algorithms and ours are nearly the same in MAB, which is a very special case of bandits learning, our analysis is different, and we prove a new result for MAB with LDP guarantee as a side-product, which achieves nearly optimal regret bound under \textit{both adversarial and stochastic environment simultaneously} (Appendix \ref{appendix: private MAB}). What's more, our results apply to more general bandits learning. For more comparison with \citet{agarwal2017price}, see Section \ref{subsec:one point feedback}. Note, even in stronger LDP context-free bandits, our framework can achieve improved regret bounds for smooth BCO compared with previous results under weaker DP guarantee \cite{thakurta2013nearly}. Besides, to the best of our knowledge, we give the first results for contextual bandits under LDP.

\section{Preliminaries}\label{sec:prel}
{\bf Notations:} $[p] = \{1,2,\cdots, p\}$. $d$ is the dimension of decision space, and $e_i$ represents $i$-th basis vector. For a vector $x$ and a matrix $M$, define $\norm{x}_M:=\sqrt{x^\top Mx}$. Given a set $\mc{W}$, we define the projection into this set as $\Pi_{\mc{W}}(\cdot)$. 


Suppose the server collects certain information from each user with data domain $\mc{C}$. $\mc{C}$ can be the range of loss values in context-free bandits, or both contexts and losses/rewards in contextual bandits. Now we define LDP rigorously:  
\begin{definition}[LDP]
    A mechanism $Q: \mc{C} \rightarrow \mc{Z}$ is said to protect $(\varepsilon, \delta)$-LDP, if for any two data $x, x' \in \mc{C}$, and any (measurable) subset $U \subset \mc{Z}$, there is 
    \begin{align*}
        \Pr[Q(x) \in U] \leqslant e^{\varepsilon} \Pr[Q(x') \in U] + \delta
    \end{align*}
    In particular, if $\mathcal{Q}$ preserves $(\varepsilon,0)$-LDP, we call it $\varepsilon$-LDP.
\end{definition}

Now, we introduce a basic mechanism in LDP literature -- Gaussian Mechanism. Given any function $h: \mc{C} \rightarrow \mb{R}^d$. Define $\Delta := \max_{x,x' \in \mc{C}} \norm{h(x)-h(x')}_2$, then Gaussian Mechanism is defined as $h(x)+ Y$, where random vector $Y$ is sampled from Gaussian distribution $\mc{N}(0, \sigma^2 \mr{I}_d)$ with $\sigma = \frac{\Delta\sqrt{2\ln(1.25/\delta)}}{\varepsilon}$. One can prove Gaussian Mechanism preserves $(\varepsilon, \delta)$-LDP \cite{dwork2014algorithmic}.

Next, we define the common strong convexity and smoothness for a function $f$. 
\begin{definition}
    We say that a function $f: \mc{X} \rightarrow \mb{R}$ is $\mu$-strongly convex if there is: $f(x) - f(y) \leqslant \nabla f(x)^\top (x-y) - \frac{\mu}{2} \norm{x-y}_2^2$.
    We say that a function $f: \mc{X} \rightarrow \mb{R}$ is $\beta$-smooth if it satisfies the following inequality: $\left| f(x) - f(y) - \nabla f(y)^\top (x-y) \right| \leqslant \frac{\beta}{2} \norm{x-y}_2^2$
\end{definition}



\section{Nearly Optimal Context-Free Bandits Learning with LDP Guarantee}
In this section, we consider private context-free bandits learning with LDP guarantee, including bandits with one-point and multi-point feedback. As the following theorem shows, LDP is much stronger than DP in this setting (see Appendix \ref{appendix: DP} for the definition of DP in streaming setting and the proof), therefore it is more difficult to design algorithms under LDP with nearly optimal guarantee.

\begin{theorem}
\label{theorem: DPLDP}
If an algorithm $\mc{A}$ protects $\varepsilon$-LDP, then any algorithm based on the output of $\mc{A}$ on a sequence of users  guarantees $\varepsilon$-DP in streaming setting.
\end{theorem}

\subsection{Private Bandits Learning with One-Point Feedback}
\label{subsec:one point feedback}
Bandits learning with one-point feedback includes several important cases, such as BCO, MAB, and Best Arm Identification (BAI). Generally speaking, we need to choose an action in the decision set at each round based on all previous information, then receive corresponding loss value of the action we choose. Most of time, our goal is to design an algorithm to minimize regret (it will be defined clearly later) compared with any fixed competitor.

Different with previous work \cite{thakurta2013nearly,mishra2015nearly,tossou2016algorithms,tossou2017achieving,sajed2019optimal}, which designed delicate algorithms for different bandit learning problems under DP, here we propose a general framework to solve all of them within a unified analysis under stronger LDP. Our general private framework is shown in Algorithm \ref{algorithm:reduction bandit}, based on a pre-chosen non-private black-box bandits learning algorithm $\mc{A}$. Definitions of $\mc{X}, f_t$ and the choice of $\mc{A}$ in Algorithm \ref{algorithm:reduction bandit} will be made clear in concrete settings below. Here we only assume all $f_t(x)$ are bounded by a constant $B$, i.e., $\forall x \in \mc{X}, t \in [T], |f_t(x)|\leqslant B$. 

For private linear bandits learning, \citet{agarwal2017price} also propose a general reduction framework that can achieve nearly optimal regret. The key idea is to inject a linear perturbation $\langle n_t, x_t \rangle$ to the observed value $f_t(x_t)$ at each round, where $x_t$ is the current decision strategy and $n_t$ is fresh \textit{noise vector} sampled from a predefined distribution. Because of the special form of linear loss, their approach actually protects data sequence in the \textit{functional} sense, i.e., it is equivalent to disturbing original linear loss function $f_t(x)$ with noisy function $n_t^\top x$. However, this approach cannot protect privacy when loss functions are nonlinear, as injected noise depends on strategy $x_t$. Just consider $x_t = 0$, then it may leak the information of $f_t$ as values of different nonlinear functions can be different at point $x_t = 0$ and there is no noise at all if we use perturbation $\langle n_t, x_t \rangle$. Instead, our main idea is to inject fresh \textit{noise variable} directly to the observed loss value at each round, which doesn't rely on $x_t$ any more. Intuitively, this approaches looks more natural as bandits learning algorithms only use the information of these observed loss values instead of loss functions. 

\begin{algorithm}[t!]
 \caption{One-Point Bandits Learning-LDP}
\label{algorithm:reduction bandit}
	  \textbf{Input}: non-private algorithm $\mc{A}$, privacy parameters $\varepsilon, \delta$ \\
	  \textbf{Initialize:} set $\sigma = \frac{ 2B\sqrt{2\ln(1.25/\delta)}}{\varepsilon}$ \\
	 \For{$t=1,2, \ldots$}{
      Server plays $x_t \in \mc{X}$ returned by $\mc{A}$;\\
      User $t$ suffers loss $f_t(x_t)$ and sends $f_t(x_t) + Z_t$ to $\mc{A}$ in the server, where $Z_t \sim \mc{N}(0, \sigma^2)$; \\
      $\mc{A}$ receives $f_t(x_t) + Z_t$ and calculates $x_{t+1}$
     }
\end{algorithm}

Obviously, the LDP guarantee of Algorithm \ref{algorithm:reduction bandit} is followed directly from basic Gaussian mechanism.  

\begin{theorem}
\label{theorem: one point privacy}
Algorithm \ref{algorithm:reduction bandit} guarantees $(\varepsilon, \delta)$-LDP.
\end{theorem} 

To show the power of Algorithm \ref{algorithm:reduction bandit}, here we consider its main application, Bandits Convex Optimization. For another two concrete applications, MAB and BAI, see Appendix \ref{appendix: examples} for more details. Besides, it also looks promising to extend the technique to pure exploration in combinatorial bandits (e.g., \cite{huang2018combinatorial}).

In bandit convex optimization \cite{hazan2016introduction}, $\mc{X}$ is a bounded convex constraint set. At each round, the server chooses a prediction $x_t$ based on previous collected information, then suffers and observers a loss value  $f_t(x_t)$. The goal is to design an algorithm with low regret defined as $\max_{x\in\mc{X}}\mb{E}[\sum_{t=1}^T f_t(x_t) - f_t(x)]$. There are two different environments which generate underlying loss function sequence $\{f_t(x)| t\in [T]\}$. For adversarial BCO, there is no further assumption about $\{f_t(x)|t\in[T]\}$ and they are fixed functions given before games starts. For stochastic BCO \cite{agarwal2011stochastic}, feedback $f_t(x_t)$ is generated as $f(x_t) + q_t$, where $f(x)$ is an unknown convex function and $\{q_t\}$ are independently and identically distributed noise sampled from a sub-Gaussian distribution $\mc{Q}$ with mean $0$.  

A critical ingredient in BCO is the gradient estimator constructed through the observed feedback. Besides convexity, when $f_t$ have additional properties like smoothness or strong convexity, usually we need to construct different gradient estimators and use different efficient non-private algorithms $\mc{A}$ to achieve better performance \cite{flaxman2005online,agarwal2010optimal,saha2011improved,hazan2014bandit}. Denote $u_t$ as a uniform random vector sampled from the unit sphere, then two representatives of gradient estimators are sphere sampling estimator $\frac{d}{\rho}f_t(x_t)u_t$ used in \cite{flaxman2005online,agarwal2010optimal} ($\rho$ is a parameter), and advanced ellipsoidal sampling estimator $d f_t(x_t)A_t^{-1}u_t$ which is the key part in \cite{saha2011improved,hazan2014bandit} to further improve the performance, where $A_t$ is the Hessian matrix induced by 
certain loss function with self-concordant barrier.

When it comes to private setting, \citet{thakurta2013nearly} designed a delicate \textit{differentially private} algorithm with $\tilde{\mc{O}}\left(T^{3/4}/\varepsilon\right)$ and $\tilde{\mc{O}}\left(T^{2/3}/\varepsilon\right)$ guarantees for convex and strongly convex loss functions respectively, based on classical sphere sampling estimator and tree-based aggregation technique \cite{dwork2010differential}. To achieve better bounds under additional smoothness assumption, it seems natural to combine their method with advanced ellipsoidal sampling estimator. However, this approach doesn't work even under DP guarantee, let alone LDP guarantee. In detail, to protect privacy, usually we need to add noise proportional to the range of information we use. For classical sphere sampling estimator, it is bounded by $dB/\rho$. However, for the advanced ellipsoidal sampling estimator, the spectral norm of inverse Hessian of self-concordant barrier (i.e., $A_t^{-1}$) can be unbounded, which makes it hard to protect privacy. Besides, tree-based aggregation techniques fail in LDP setting.

Instead of adding noise to the accumulated estimated gradient like \citet{thakurta2013nearly}, our general reduction Algorithm \ref{algorithm:reduction bandit} injects noise directly to the loss value that is already bounded. Based on the critical observation that the regret defined for original loss functions $\{f_t(x) | t \in [T]\}$ equals to the regret defined for virtual loss functions $\{f_t(x) + Z_t | t \in [T]\}$ in expectation, we avoid complex analysis which is based on a connection with non-private solutions \cite{thakurta2013nearly}, and obtain the utility of our private algorithm through the guarantee of non-private algorithm $\mc{A}$ directly as the following shows: 

\begin{theorem}
\label{theorem: BCO utility}
Suppose non-private algorithm $\mc{A}$ achieves regret $B\cdot \Reg^T_{\mc{A}}$ for BCO, where $B$ is the range of loss function. We have the following guarantee for Algorithm \ref{algorithm:reduction bandit}: for any $x\in\mc{X}$, there is
\begin{equation}
    \mb{E}\left[\sum_{t=1}^T f_t(x_t) - f_t(x)\right] \leqslant \tilde{\mc{O}}\left(\frac{B\ln (T/\delta)}{\varepsilon}\cdot\Reg^T_{\mc{A}}\right)
\end{equation}
where expectation is taken over the randomness of non-private algorithm $\mc{A}$ and all injected noise.\footnote{Actually, if using the high probability guarantee of black-box algorithm $\mc{A}$, we can also obtain corresponding high probability guarantee of our Algorithm \ref{algorithm:reduction bandit}. See Appendix \ref{appendix:proofs in section 3} for more details, and the same argument there can be extended to results in section \ref{section:multi-point feedback} as well.}
\end{theorem} 

With above theorem, by plugging different non-private optimal algorithms under variant cases, we obtain corresponding regret bounds with LDP guarantee:

\begin{corollary}
\label{cor: concrete BCO utility}

When loss functions are convex and $\beta$-smooth, Algorithm \ref{algorithm:reduction bandit} achieves $\tilde{\mc{O}}(T^{2/3}/\varepsilon)$ regret by setting $\mc{A}$ as Algorithm 1 in \cite{saha2011improved}. When loss functions are $\mu$-strongly convex and $\beta$-smooth, Algorithm \ref{algorithm:reduction bandit} achieves $\tilde{\mc{O}}(\sqrt{T}/\varepsilon)$ regret by setting $\mc{A}$ as Algorithm 1 in \cite{hazan2014bandit}. For private Stochastic BCO, using Algorithm 2 in \cite{agarwal2011stochastic} as the black-box algorithm will achieve  $\tilde{\mc{O}}(\sqrt{T}/\varepsilon)$ regret.
\end{corollary}

Note this result improves previous result \cite{thakurta2013nearly} in three aspects. First, our Algorithm \ref{algorithm:reduction bandit} guarantees stronger LDP rather than DP. Second, it achieves better regret bounds when loss functions are further smooth, and matches corresponding non-private results. Third, our algorithm is easy to be implemented, admits a unified analysis, and also obtains new results in stochastic BCO.  


\subsection{Private Bandits Convex Optimization with Multi-Point Feedback}
\label{section:multi-point feedback}
Now we consider BCO with Multi-Point Feedback. Different with one-point bandit feedback setting, where we can only query one point at each round, now we can query multiple points. This is natural in many applications, such as in personalized recommendation, we can recommend multiple items to each user and receive their feedback. Suppose we are permitted to query $K$ points per round (denote them as $x_{t,1}, \dots, x_{t,K}$ at round $t$), then we observe $f_t(x_{t,1}), \dots, f_t(x_{t, K})$. Suppose decision set $\mc{X}$ satisfies $r\mb{B} \subset \mc{X} \subset R\mb{B}$ like in \citet{agarwal2010optimal}, where $\mb{B}$ is the unit ball in $\mb{R}^d$. The expected regret is defined as  
\begin{align}
    \mb{E}\left[\frac{1}{K}\sum_{t=1}^T\sum_{k=1}^K f_t(x_{t,k})\right] - \min_{x\in \mc{X}} \mb{E}\left[\sum_{t=1}^T f_t(x)\right]
\end{align}
where $\{f_t(x)\}$ are $G$-Lipschitz convex functions, and expectation is taken over the randomness of algorithm. 

With the relaxation of amount about queries, there is a significant difference about regret bound of BCO between one-point feedback and $K$-point feedback for $K \geqslant 2$ \cite{agarwal2010optimal}. In detail, the minimax regret for general BCO with one-point feedback is in order $\tilde{\mc{O}}(\sqrt{T})$ (even for strongly convex and smooth losses \cite{shamir2013complexity}), whereas one can design algorithms for BCO under multi-point feedback with $\mc{O}(\sqrt{T})$ regret for convex loss and $\mc{O}(\log T)$ regret for strongly convex loss, just like full information online convex optimization. As there is not much difference between $K=2$ and $K>2$, so we focus on $K=2$ in this paper. An optimal non-private algorithm can be found in \cite{agarwal2010optimal} and is given as Algorithm \ref{algorithm:non-private two point bandit} in Appendix \ref{appendix: non-private MP-BCO} for completion, which will be used as our black-box algorithm later.


For private version of this problem, note our previous reduction framework no longer fits in this new setting, mainly because of multiple feedback. If we add the same noise $Z_t$ to observed values $f_t(x_{t,1}), f_t(x_{t,2})$, then it cannot guarantee privacy. If we use different noise $Z_{t,1}, Z_{t,2}$ to perturb observed values respectively, though it protects privacy, previous utility analysis fails. 

Based on the non-private algorithm, we design a slightly modified reduction framework that resembles the approach in \citet{agarwal2017price} but for Multi-Point BCO, as shown in Algorithm \ref{algorithm:reduction two point bandit}. The key observation is that now we play two pretty close points $x_{t,1}, x_{t,2}$ at each round, and critical information we use about user $t$ is only the difference $f_t(x_{t,1}) - f_t(x_{t,2})$ of two observed values. Note $x_{t,1} - x_{t,2} = 2\rho u_t$ (see Algorithm \ref{algorithm:non-private two point bandit} in Appendix \ref{appendix: non-private MP-BCO}), which implies we can add noise $n_t^\top (x_{t,1} - x_{t,2})$ to $f_t(x_{t,1}) - f_t(x_{t,2})$ to protect its privacy. As $f_t(x)$ is $G$-Lipschitz, hence $\abs{f_t(x_{t,1}) - f_t(x_{t,2})} \leqslant 2\rho G \norm{u_t}_2$ and adding Gaussian noise with standard deviation $\sigma = \frac{2 G\sqrt{2\ln(1.25/\delta)}}{\varepsilon}$  is enough to protect privacy as $\norm{u_t}_2 = 1$. 

\begin{algorithm}[t!]
 \caption{Two-Point Feedback Private Bandit Convex Optimization via Black-box Reduction}
\label{algorithm:reduction two point bandit}
	 \textbf{Input}: set $\mc{A}$ as Algorithm  \ref{algorithm:non-private two point bandit} (in Appendix \ref{appendix: non-private MP-BCO}) with parameters $\eta, \rho, \xi$, privacy parameters $\varepsilon, \delta$ \\
	 \textbf{Initialize:} set $\sigma = \frac{ 2G\sqrt{2\ln(1.25/\delta)}}{\varepsilon}$, $\eta = \frac{1}{\sqrt{T}}, \rho = \frac{\log T}{T}, \xi = \frac{\rho}{r}$ \\
	\For{$t=1,2, \ldots$}{
	 Server plays $x_{t,1}, x_{t,2} \in \mc{X}$ received from $\mc{A}$\\
 	 User suffers $f_t(x_{t,1}), f_t(x_{t,2})$ and passes $f_t(x_{t,1}) - f_t(x_{t,2}) + n_t^\top (x_{t,1} - x_{t,2})$ to $\mc{A}$ in the server, where $n_t \sim \mc{N}(0, \sigma^2 \mr{I}_d)$ \\
    }
\end{algorithm}

\begin{theorem}
\label{two point privacy}
Algorithm \ref{algorithm:reduction two point bandit} guarantees $(\varepsilon, \delta)$-LDP. 
\end{theorem} 

For utility analysis of Algorithm \ref{algorithm:reduction two point bandit}, as now the noise depends on strategies $x_{t,1}, x_{t,2}$ at round $t$, hence both output and regret in terms of original loss functions $\{f_t(x) | t \in [T]\}$ are the same as output and regret in terms of virtual loss functions $\{f_t(x) + n_t^\top x | t \in [T]\}$ in expectation. Therefore we can obtain the utility of our private Algorithm \ref{algorithm:reduction two point bandit} through the guarantee of non-private algorithm $\mc{A}$:

\begin{theorem}
\label{theorem: two point utility}
For any $x \in \mc{X}$, Algorithm \ref{algorithm:reduction two point bandit} guarantees 
\begin{equation}
    \mb{E}\left[\frac{1}{2}\sum_{t=1}^T \left(f_t(x_{t,1}) + f_t(x_{t,2})\right) - f_t(x)\right] \leqslant \tilde{\mc{O}}\left(\frac{d^3\sqrt{T}}{\varepsilon^2}\right)
\end{equation}
If $\{f_t\}$ are further $\mu$ strongly convex, set $\eta = \frac{1}{\mu t}, \rho = \frac{\log T}{T}, \xi = \frac{\rho}{r}$, then for any $x \in \mc{X}$, we have 
\begin{equation}
    \mb{E}\left[\frac{1}{2}\sum_{t=1}^T \left(f_t(x_{t,1}) + f_t(x_{t,2})\right) - f_t(x)\right] \leqslant \tilde{\mc{O}}\left(\frac{d^3\log T}{\mu\varepsilon^2}\right)
\end{equation} 
\end{theorem}

From above results, one can see there is also a significant difference about regret bounds between BCO and Multi-Point BCO under LDP setting, which is exactly the same as non-private settings.  

\section{Contextual Bandits Learning with LDP Guarantee}
\label{sec: private contextual bandits}
In this section, we turn our attention to more practical contextual bandits learning. At each round $t$, the learner needs to choose an action $x_t \in \mc{X}_t$ in the local side, where $\mc{X}_t$ contains the personal information and features about underlying arms. Then the user generates a reward which is assumed to be $y_t = g(x_t^\top \theta^*)+\eta_t$, where $\theta^*$ is an unknown true parameter in the domain $\mc{W}$ , $g:\mb{R}\rightarrow \mb{R}$ is a known function, and $\eta_t$ is a random noise in $[-1,1]$ with mean 0 \footnote{It's not hard to relax this constraint to a sub-Gaussian noise.}. If we know $\theta^*$, $x_{t,*}:=\argmax_{x\in \mc{X}_t} g(x^\top \theta^*)$ is apparently the optimal choice at round $t$. For an algorithm $\mc{A}$, we define its regret over $T$ rounds as $\mr{Reg}_T^{\mc{A}}:= \sum_{t=1}^T g(x_{t,*}^\top \theta^*) - g(x_t^\top \theta^*)$, where $\{x_t, t\in [T]\}$ is the output of $\mc{A}$. We omit the superscript $\mc{A}$ when it is clear. There are two critical parts in contextual bandits. One is to estimate $\theta^*$, and corresponding estimated parameter is used to find best action for exploitation. Another one is to construct certain term for the purpose of exploration, since we are in the environment of partial feedback. Throughout this section, we assume both $\{\mc{X}_t\}$ and $\mc{W}$ are bounded by a $d$-dimensional $L_2$ ball with radius $1$ for simplicity.  

Compared with private context-free bandits, private contextual bandits learning is more difficult, not only because of relatively complicated setting, but we need to protect more information including both contexts and rewards, which causes additional difficulty in the analysis of regret. As a warm-up, we show how to design algorithm with LDP guarantee for contextual linear bandits, which resembles a recent work \cite{shariff2018differentially} but under a relaxed version of DP. Next, we propose a more complicated algorithm for generalized linear bandits with LDP guarantee.

\subsection{Warm-Up: LDP Contextual Linear Bandits} 
In contextual linear bandits, mapping $g$ is an identity, or equivalently, the reward generated by user $t$ for action $x_t$ is $y_t = x_t^\top \theta^* + \eta_t$. To estimate $\theta^*$, the straightforward method is to use linear regression based on collected data. Combined with classic principal for exploration, optimism in the face of uncertainty, it leads to LinUCB \cite{abbasi2011improved}, which is nearly optimal for contextual linear bandits. To protect privacy, it's not surprising that we adopt the same technique as LDP linear regression \cite{smith2017interaction}, i.e. injecting noise to $x_t x_t^\top$ and $y_t x_t$ collected from user $t$. However, the injected noise have influence not only over the parameter estimation, but also for further exploration part, due to more complex bandit model, thus we need to set parameters more carefully. See Algorithm \ref{algorithm: LDP linear bandit} in Appendix \ref{appendix: LDP linear bandit}.

Now, we state the theoretical guarantee of Algorithm \ref{algorithm: LDP linear bandit}.

\begin{theorem}
\label{theorem: linear bandit privacy}
Algorithm \ref{algorithm: LDP linear bandit} guarantees $(\varepsilon, \delta)$-LDP. 
\end{theorem}

\begin{theorem}
\label{theorem: linear bandit utility}
With probability at least $1-\alpha$, the regret of Algorithm \ref{algorithm: LDP linear bandit} satisfies the following bound: 
\begin{align}
	\mr{Reg}_T \leqslant \tilde{\mc{O}}\left(\sqrt{\log \frac{1}{\delta} \log \frac{1}{\alpha}} \frac{(dT)^{3/4}}{\varepsilon}\right)
	\label{equation: LB bound}
\end{align}
\end{theorem}

Given the $\Omega(T)$ lower bound for DP contextual linear bandits \cite{shariff2018differentially}, Theorem \ref{theorem: linear bandit utility} implies a fundamental difference between LDP and DP in contextual bandit learning, which also verifies that LDP is a more appropriate standard about privacy for contextual bandits as discussed in the introduction. One may think we can still prove DP based on LDP guarantee and post-processing property. Recall post-processing property holds only for the output of a DP algorithm which doesn’t use private data any more. However, in our algorithms for LDP contextual bandits, though we can use post-processing property to prove estimation sequence $\{\tilde{\theta}_t\}$ satisfies DP, it doesn’t imply the output action sequence $\{x_t\}$ satisfies DP, as these actions are made in the local side which use private local data.

\subsection{LDP Generalized Linear Bandits}
In generalized linear bandits, mapping $g$ can be regarded as the inverse link function of exponential family model. Here we suppose function $g$ is $G$-Lipschitz, continuously differentiable on $[-1,1]$, $|g(a)|\leqslant C$, and $\inf_{a\in (-1,1)} g'(a) = \mu > 0$, which implies $g$ is strictly increasing. These assumptions are common either in real applications or previous work \cite{li2017provably,jun2017scalable}. We also define corresponding negative log-likelihood function $\ell(a,b):=-ab+m(a)$, where $m(\cdot)$ is the integral of function $g$. As a concrete example, if reward $y$ is a Bernoulli random variable, then the form of $g$ is $g(a)=(1+\exp(-a))^{-1}$, $m(a)=\log(1+\exp(a))$, $\ell(a,b)=\log(1+\exp(-a(2b-1))), b\in \{0,1\}$, and noise $\eta$ is $1-g(a)$ with probability $g(a)$ and $-g(a)$ otherwise.

Note the non-linearity of $g$ makes things much more complicated either from the view of bandits learning or privacy preservation. The counterpart of Contextual Linear Bandits is linear regression, the locally private version of which is relatively easy and well-studied. However, the counterpart of Generalized Linear Bandit is Empirical Risk Minimization (ERM) with respect to generalized linear loss, and the optimal approach of parameter estimation for GLM bandit is to solve ERM at each round \cite{li2017provably}. Different with linear regression, for learning ERM with LDP guarantee, in general there is no efficient private algorithm that can achieve optimal performance in the \textit{non-interactive} environment \cite{smith2017interaction,zheng2017collect,wang2018empirical}, let alone calculating an accurate parameter estimation needed in our problem. Therefore, it seems hard to learn generalized linear bandit under LDP guarantee. 

Luckily, we can make full use of the \textit{interactive} environment in bandit problems. In detail, 
we build our private mechanism based on GLOC framework proposed in \cite{jun2017scalable}. Compared with previous nearly optimal approach \cite{li2017provably}, GLOC framework enjoys much better time efficiency, which calculates estimator $\theta$ in an online fashion instead of solving ERM at each round. Its main idea is to maintain a rough estimation for unknown parameter $\theta^*$ through an adversarial online learning algorithm and use it to relabel current reward, 
and then solve the corresponding linear regression for a refined estimator. To achieve optimal $\tilde{\mc{O}}(\sqrt{T})$ regret, the online learning algorithm is set as Online Newton Step \cite{hazan2007logarithmic}.

Though the original goal of GLOC framework proposed in \citet{jun2017scalable} is to improve time efficiency, the update form of estimated parameter for unknown $\theta^*$ shares the same form of linear regression, therefore we can use nearly the same technique as in previous subsection to protect LDP, which avoids solving complex ERM with LDP guarantee. Besides, since internal online learning algorithm also utilizes users' data, we also need to guarantee its privacy. Different with \citet{jun2017scalable} which adopts Online Newton Step, we choose basic noisy Online Gradient Descent as our online black-box algorithm. See Algorithm \ref{algorithm: GLM bandit} for the full implementation. For clarity, we just write the LDP Online Gradient Descent explicitly in Line 11. 

\begin{algorithm}[t!]
 \caption{Generalized Linear Bandits with LDP}
\label{algorithm: GLM bandit}
  \textbf{Input:} privacy parameters $\varepsilon, \delta$, failure probability $\alpha$\\
  \textbf{Initialize:} $\tilde{V}_0 = 0_{d\times d}, \tilde{u}_0 = 0_d, \tilde{\theta}_0 = \hat{\theta}_1=0_d$, $\zeta = \Theta(1/\sqrt{T})$, $\sigma = 6\sqrt{2\ln(3.75/\delta)}/\varepsilon$\\
  \textbf{Notations:} $\Upsilon_t = \sigma \sqrt{t}(4\sqrt{d} + 2\ln(2T/\alpha)), c_t = 2\Upsilon_t$, $\beta^2_t = \tilde{\mc{O}}(\frac{C\sigma}{\mu}\sqrt{dt})$ 

\For{$t=1,2, \ldots$}{
 \textbf{For the local user $t$:}\\
 Receive information $\tilde{V}_{t-1}, \tilde{\theta}_{t-1}, \hat{\theta}_t$ from the server\\
 Play action $x_t = \argmax_{x \in \mc{D}_t} \inner{\tilde{\theta}_{t-1}, x} + \beta_{t-1} \norm{x}_{\tilde{V}^{-1}_{t-1}}$\\
 Observe reward $y_t = g(x_t^\top\theta^*) + \eta_t$, set $z_t = x_t^\top \hat{\theta}_t$. \\
 Send $x_tx_t^\top + B_t, z_t x_t + \xi_t, \nabla \ell_t(\hat{\theta}_t) + r_t$ to the server, where $\ell_t(\theta)= \ell(x_t^\top\theta, y_t), B_t(i,j) \overset{i.i.d}{\sim} \mc{N}(0, \sigma^2), \forall i \leqslant j$, and $B(j,i) = B(i,j), \xi_t \sim \mc{N}(0_d, \sigma^2 \mr{I}_{d\times d}), r_t \sim \mc{N}(0_d, C^2 \sigma^2 \mr{I}_{d\times d})$\\
 \textbf{For the server:} \\
 Update $\bar{V}_{t} = \bar{V}_{t-1} + x_tx_t^\top + B_t, \tilde{u}_{t} = \tilde{u}_{t-1} + z_t x_t + \xi_t$
 $\tilde{\theta}_t = \tilde{V}_t^{-1} \tilde{u}_t$, where $\tilde{V}_t = \bar{V}_t+c_t \mr{I}_{d\times d}$
 $\hat{\theta}_{t+1} = \Pi_{\mc{W}}\left(\hat{\theta}_t - \zeta (\nabla \ell_t(\hat{\theta}_t) + r_t)\right)$ \label{alg-step: LDP_OGD} \label{step:GLBandits-OGD}
}
\end{algorithm}

Though Algorithm \ref{algorithm: GLM bandit} is based on the framework proposed by \citet{jun2017scalable}, we want to emphasize that both finding the right approach and proving the rigorous guarantee are non-trivial because of stringent LDP constraint. Following theorems give both the privacy guarantee and utility bound of our Algorithm \ref{algorithm: GLM bandit} for generalized linear bandits. 

\begin{theorem}
\label{theorem: gl bandit privacy}
Algorithm \ref{algorithm: GLM bandit} guarantees $(\varepsilon, \delta)$-LDP. 
\end{theorem}

\begin{theorem}
\label{theorem: gl bandit utility}
With probability at least $1-\alpha$, the regret of Algorithm \ref{algorithm: GLM bandit} satisfies the following bound: 
\begin{align}
	\mr{Reg}_T \leqslant \tilde{\mc{O}}\left(\sqrt{\log \frac{1}{\delta} \log \frac{1}{\alpha} \log \frac{T}{d}}\frac{(dT)^{3/4}}{\varepsilon}\right)
	\label{equation: GLB bound}
\end{align}
\end{theorem}

Note that both our upper bounds (\ref{equation: LB bound}) and (\ref{equation: GLB bound}) are in order $\tilde{\mc{O}}\left(T^{3/4}\right)$, which differ from common $\mc{O}(\sqrt{T})$ regret bound in corresponding non-private settings. We conjecture this order is nearly the best one can achieve in LDP setting, mainly because we need to protect more information, i.e., both contexts and corresponding rewards. See Appendix \ref{appendix: lower bound} for more discussions and intuitions.

\section{Conclusions}
In this paper, we propose a simple black-box reduction framework that can solve a large class of context-free bandits learning problems with LDP guarantee in a unified way, including BCO, MAB, Best Arm Identification. We also extend the reduction framework to BCO with Multi-Point Feedback. This black-box reduction mainly has three advantages compared with previous work. First it guarantees a more rigorous LDP guarantee instead of DP. Second, this framework gives us a unified analysis for all above private bandit learning problems instead of analyzing each of them separately, and it easily improves previous best results or obtains new results for some problems, as well as matching corresponding non-private optimal bounds. Third, such a black-box reduction is more attractive in real applications, as we only need to modify the input to black-box algorithms. Besides, we also propose new algorithms for more practical contextual bandits with LDP guarantee, including contextual linear bandits and generalized linear bandits. Our algorithms can achieve $\tilde{\mc{O}}(T^{3/4})$ regret bound, which is conjectured to be nearly optimal. We leave the rigorous proof of this lower bound as an interesting open problem. 


\textbf{Broader Impact}\\
This work is mostly theoretical, with no negative outcomes. (Contextual) bandits learning has been widely used in real applications, which heavily relies on user's data that may contain personal private information. To protect user's privacy, we adopt the appealing solid notion of privacy -- Local Differential Privacy (LDP) that can protect each user's data before collection, and design (contextual) bandit algorithms under the guarantee of LDP. Our algorithms can be easily used in real applications, such as recommendation, advertising, to protect data privacy and ensure the utility of private algorithms simultaneously, which will befit everyone in the world.

\begin{ack}
This work was supported by National Key R\&D Program of China (2018YFB1402600), Key-Area Research and Development Program of Guangdong Province (No. 2019B121204008), Beijing Academy of Artificial Intelligence, and in part by the Zhongguancun Haihua Institute for Frontier Information Technology. 
\end{ack}

{\small
\bibliography{reference} 
\bibliographystyle{abbrvnat}
}
\clearpage
\appendix
\section*{Appendix}
\addcontentsline{toc}{section}{Appendices}
\renewcommand{\thesubsection}{\Alph{subsection}}

\subsection{Differential Privacy under streaming setting}
\label{appendix: DP}
Differential Privacy \cite{dwork2006calibrating} is original proposed for off-line setting. Later, \citet{dwork2010differential} and \citet{jain2012differentially} consider DP in streaming setting. In streaming setting, at each round $t$, the server predicts $x_t\in \mc{X}$ for user $t$ whose personal data is represented as $h_t \in \mc{H}$ (for example, his or her feature, label, or preference etc.). Then the server requires some information $z_t \in \mc{Z}$ from user $t$ ($z_t$ may depend on $x_t$ and $h_t$) to update the model for next prediction. Note DP allows collecting true data (i.e. $z_t = h_t$) and is defined in terms of the output sequence $\{x_t\}$, while LDP doesn't allow collecting true data and is defined in terms of the collected information $z_t$. Here we adopt the definition given in \citet{jain2012differentially} for DP in streaming setting:

\begin{definition}[Differential Privacy]
Let $F = \langle h_1, h_2, \dots, h_T \rangle$ be a sequence of information which domain is $\mc{H}^{1:T}$. Let $\mathcal{A}(F) = Y$, where $Y = \langle y_1, y_2, \dots, y_T \rangle \in \mathcal{Y}^{1:T}$ be $T$ outputs of the randomized algorithm $\mathcal{A}$. $\mathcal{A}$ is said to preserve $(\varepsilon, \delta)$-differential privacy, if for any two information sequences $F, F'$ that differ in at most one entry, and for any subset $S^{1:T} \subset \mathcal{Y}^{1:T}$, it holds that
\[
  \Pr(\mathcal{A}(F)\in S^{1:T}) \leq \Pr(\mathcal{A}(F')\in S^{1:T})e^\varepsilon + \delta.
\]
In particular, if $\mathcal{A}$ preserves $(\varepsilon,0)$-differential privacy, we say $\mathcal{A}$ is $\varepsilon$-differentially private.
\end{definition}

Now, we prove Theorem \ref{theorem: DPLDP}:
\begin{proof}[Proof of Theorem~\ref{theorem: DPLDP}]
    Suppose algorithm $\mc{A}: \mc{H} \rightarrow \mc{Z}$ protects $\varepsilon$-LDP, that is for any $h,h'\in\mc{H},U\subset \mc{Z}$, we have 
    \begin{align*}
          \Pr(\mathcal{A}(h)\in U) \leqslant e^{\varepsilon}\times \Pr(\mathcal{A}(h')\in U)
      \end{align*}  
    Denote $\mc{G}$ as arbitrary online/bandits algorithm received the output of $\mc{A}$ on user sequence, i.e. $\{z_t= \mc{A}(h_t|x_t) | t\in [T]\}$. Now we prove $\mc{G}$ protects $\varepsilon$-DP, i.e. for any $S^{1:T} \subset \mc{X}^{1:T}$ and neighboring sequence $F=\{h_t| t\in [T]\}, F'=\{h'_t| t\in [T]\}$ that only differ in one entry, we have the following inequality:
    \begin{align*}
          \Pr(\mc{G}(\mathcal{A}(F)) \in S^{1:T}) \leqslant e^{\varepsilon}\times \Pr(\mc{G}(\mathcal{A}(F')) \in S^{1:T})
      \end{align*} 
      Without loss of generality, we assume $F$ and $F'$ differ in the $t$-th entry. Since $\mc{G}$ only operates on $\{z_t|t\in [T]\}$, according to the Post-Processing property of DP \cite{dwork2014algorithmic}, we only need to prove $\{z_t|t\in [T]\}$ satisfies $\varepsilon$-DP. Denote $\{z'_t|t\in [T]\}$ as the neighboring information sequence of $\mc{A}$ operated on $F'$, then for arbitrary $U^{1:T} \subset \mc{Z}^{1:T}$ we have
      \begin{align}
          & \frac{\Pr(z_{1:T} \in U^{1:T})}{\Pr(z'_{1:T} \in U^{1:T})} \\
          = & \frac{\Pr(z_{1:t-1} \in U^{1:t-1})\times \Pr(z_t \in U^t |z_{1:t-1} \in U^{1:t-1}) \times \Pr(z_{t+1:T} \in U^{t+1:T}|z_{1:t} \in U^{1:t})}{\Pr(z'_{1:t-1} \in U^{1:t-1})\times \Pr(z'_t \in U^t |z'_{1:t-1} \in U^{1:t-1}) \times \Pr(z'_{t+1:T} \in U^{t+1:T}|z'_{1:t} \in U^{1:t})} \\
          = & \frac{\Pr(z_t \in U^t |z_{1:t-1} \in U^{1:t-1})}{\Pr(z'_t \in U^t |z'_{1:t-1} \in U^{1:t-1}) } \\
          = & \frac{\Pr(z_t \in U^t |x_t \in \mc{G}(U^{1:t-1}))}{\Pr(z'_t \in U^t |x'_t \in \mc{G}(U^{1:t-1})) } \\
          \leqslant & e^{\varepsilon}
      \end{align}
      where the second equation is because two data sequence only differ at round $t$, and $\mc{G}$ operates on the sequence of $z$. Thus we prove the theorem.
\end{proof}

\subsection{Another Two Applications for Bandits Learning with One-point Feedback}
\label{appendix: examples}
\subsubsection{Private Multi-Armed Bandits}
\label{appendix: private MAB}
MAB is a special case of BCO, in which decision set $\mc{X} = \{e_i|i\in[d]\}$, and loss function $f_t(x)$ is actually a linear function, i.e. $f_t(x) = \ell_t^\top x$, where $\ell_t \in [0,1]^d$ . In the adversarial setting, sequence $\{\ell_t\}$ is chosen arbitrarily before game starts. In stochastic setting, for each arm $k$, $\{\ell_t(k)\}$ are independently sampled from underlying unknown distribution $\mc{V}_k$ with support over interval $[0,1]$. Denote $\mu_k$ as the expected loss of arm $k$. Without loss of generality, assume $\mu_1 > \mu_2 > \dots > \mu_d$ and define $\Delta_i := \mu_i - \mu_d$. It is well-known the optimal regret are $\mc{O}(\sqrt{dT})$ and $\mc{O}(\sum_{i:\Delta_i>0}\frac{\log T}{\Delta_i})$ for adversarial MAB and stochastic MAB respectively \cite{bubeck2012regret}. However, especially in real applications, usually we don't know whether we are in adversarial or stochastic environment in advance. Until recently, \citet{zimmert19a} proposed a single algorithm achieving the optimal performance for both adversarial and stochastic world \textit{without} any prior information about the environment.

For differentially private MAB, all of previous work consider either stochastic loss \textit{or} adversarial loss \cite{mishra2015nearly, tossou2016algorithms,tossou2017achieving,agarwal2017price}. While here, we hope to handle both scenarios simultaneously like in non-private case but with LDP guarantee. Not surprisingly, by plugging the non-private optimal algorithm \cite{zimmert19a} in our black-box, we obtain corresponding private version which achieves the best of both adversarial and stochastic worlds:

\begin{theorem}
 \label{theorem: private MAB both worlds}
By choosing non-private black-box algorithm $\mc{A}$ in Algorithm \ref{algorithm:reduction bandit} as \textsc{Tsallis-Inf} in \citet{zimmert19a} and setting $\sigma$ as in Theorem \ref{theorem: one point privacy} with $B=0.5$, 
\begin{itemize}
    \item in the adversarial setting, we have 
    \begin{align}
        \max_{x \in \mc{X}}\mb{E}\left[\sum_{t=1}^T \ell_t(x_t) - \ell_t(x)\right] \leqslant \tilde{\mc{O}}\left(\frac{\sqrt{T}}{\varepsilon}\right)
    \end{align} 
    \item in the stochastic setting, we have 
    \begin{align}
        \max_{x \in \mc{X}}\mb{E}\left[\sum_{t=1}^T \ell_t(x_t) - \ell_t(x)\right] \leqslant \tilde{\mc{O}}\left(\sum_{i:\Delta_i>0}\frac{\log T}{\Delta_i \varepsilon^2}\log \frac{1}{\delta}\right)
    \end{align}
\end{itemize}
\end{theorem}

Note above results not only nearly match corresponding non-private lower bounds \cite{bubeck2012regret} regardless of privacy parameters, but also lower bounds under LDP restriction \cite{basu2019differential}. Besides, we can also use many other MAB algorithms as our black-box candidates such as KL-UCB \cite{garivier2011kl} Stochastic MAB, which will then obtain more delicate bound under LDP.

\subsubsection{Private Best Arm Identification}
\label{appendix: private BAI}
Different with Stochastic MAB, in which one has to balance between \textit{Exploration} and \textit{Exploitation}, Best Arm Identification (BAI) problem only focuses on the \textit{Exploration}, that is finding the best arm among all arms. Here we use same notations as Subsection \ref{appendix: private MAB}. There are mainly two settings in BAI: \textit{fixed confidence setting} and \textit{fixed budget setting}. In this part, we only consider \textit{fixed confidence setting}: given any confidence parameter $\gamma$, design an algorithm which outputs the best arm with probability at least $1-\gamma$ using as fewest samples as possible \cite{jamieson2014best,kaufmann2016complexity}. It's not hard to see our method can be generalized to fixed budget setting as well. 

For private BAI, though algorithms in \citet{mishra2015nearly} and \citet{sajed2019optimal} are designed for stochastic MAB, they can also used for differentially private BAI. However, these algorithms only achieve sub-optimal guarantee, let alone stronger LDP. While here, we want to protect LDP and achieve nearly optimal sample complexity. Again, using the same observation as Subsection \ref{appendix: private MAB} and given any non-private BAI algorithm $\mc{A}$, our Algorithm \ref{algorithm:reduction bandit} has the following guarantee:

\begin{theorem}
\label{BAI utility}
 Given any confidence parameter $\gamma$, suppose non-private BAI algorithm $\mc{A}$ achieves sample complexity $\SA(\mc{A}, \sigma_0^2, \gamma)$, where $\sigma_0^2$ is the variance proxy parameter of underlying unknown sub-Gaussian distributions $\{\mc{V}_k | k\in [d]\}$. Set $\sigma$ as in Theorem \ref{theorem: one point privacy} with $B=0.5$, then the sample complexity of Private BAI Algorithm \ref{algorithm:reduction bandit} is $\SA(\mc{A}, \frac{1}{4} + \sigma^2, \gamma)$. 

Specifically, if we choose non-private BAI algorithm $\mc{A}$ as lil'UCB in \citet{jamieson2014lil}, then the sample complexity of Algorithm \ref{algorithm:reduction bandit} is in order $\mc{O}\left(\sum_{k\ne 1}\frac{\ln\left(\left(\ln 1/\Delta_k^2\right)/\gamma\right)}{\varepsilon^2\Delta^2_k}\ln \frac{1}{\delta}\right)$. 
\end{theorem}

\subsection{Non-private Algorithm for Bandits Learning with Two-points Feedback}
\label{appendix: non-private MP-BCO}
For completeness, we present the non-private algorithm proposed in \citet{agarwal2010optimal} for Bandits Convex Optimization with two-point feedback. See Algorithm \ref{algorithm:non-private two point bandit}. 
\begin{algorithm}[t!]
 \caption{Expected Gradient Descent with two queries per round \cite{agarwal2010optimal}}
\label{algorithm:non-private two point bandit}
      \textbf{Input}: Learning rate $\eta$, exploration parameter $\rho$ and shrinkage coefficient $\xi$ \\
      Set $y_1 = 0$\\
     \For{$t=1,2, \ldots$}{
      Pick a unit vector $u_t$ uniformly at random \\
      Play $x_{t,1}:=y_t + \rho u_t, x_{t,2}:=y_t - \rho u_t$, and observe $f_t(x_{t,1}), f_t(x_{t,2})$\\
      Set $\tilde{g}_t = \frac{d}{2\rho}\left(f_t(x_{t,1}) - f_t(x_{t,2})\right)u_t$\\
      update $y_{t+1} = \prod_{(1-\xi)\mc{X}}(y_t - \eta \tilde{g}_t)$, where $\prod_{\mc{X}}$ represents projection to the set $\mc{X}$
     }
\end{algorithm}

\subsection{Contextual Linear Bandits with LDP}
\label{appendix: LDP linear bandit}
See Algorithm \ref{algorithm: LDP linear bandit} above.
\begin{algorithm}[t!]
 \caption{Contextual Linear Bandits with LDP}
\label{algorithm: LDP linear bandit}
  \textbf{Input:} privacy parameters $\varepsilon, \delta$, failure probability $\alpha$.\\
  \textbf{Initialize:} $\tilde{V}_0 = 0_{d\times d}, \tilde{u}_0 = 0_d, \tilde{\theta}_0 = 0_d,$
 $\sigma = 6\sqrt{2\ln(2.5/\delta)}/\varepsilon$.\\
  \textbf{Notations:} $\Upsilon_t = \sigma \sqrt{t}(4\sqrt{d} + 2\ln(2T/\alpha)), c_t = 2\Upsilon_t$, $\beta_t = 2\sigma \sqrt{d \ln T} + \left(\sqrt{3\Upsilon_t} + \sigma\sqrt{\frac{dt}{\Upsilon_t}}\right)d\ln T$. 

\For{$t=1,2, \ldots, T$}{
 \textbf{For the local user $t$:}\\
 Receive information $\tilde{V}_{t-1}, \tilde{\theta}_{t-1}$ from the server.\\
 Play action $x_t = \argmax_{x \in \mc{D}_t} \inner{\tilde{\theta}_{t-1}, x} + \beta_t \norm{x}_{(\tilde{V}_{t-1} + c_{t-1}\mr{I})^{-1}}$\\
 Observe reward $y_t = \inner{x_t, \theta^*} + \eta_t$ \\
 Send $x_tx_t^\top + B_t, y_t x_t + \xi_t$ to the server, where $B_t(i,j) \overset{i.i.d}{\sim} \mc{N}(0, \sigma^2), \forall i \leqslant j$, and $B(j,i) = B(i,j), \xi_t \sim \mc{N}(0_d, \sigma^2 \mr{I}_{d\times d})$.\\
 \textbf{For the server:} update\\
 $\tilde{V}_{t} = \tilde{V}_{t-1} + x_tx_t^\top + B_t, \tilde{u}_{t} = \tilde{u}_{t-1} + y_t x_t + \xi_t$\\
 $\tilde{\theta}_t = \left(\tilde{V}_t+c_t \mr{I}_{d\times d}\right)^{-1} \tilde{u}_t$
}
\end{algorithm}

\subsection{Omitted Proofs in Section 3}
\label{appendix:proofs in section 3}

\begin{proof}[Proof of Theorem~\ref{theorem: one point privacy}]
Since for any $x \in \mc{X}, t\in [T]$, $|f_t(x)|\leqslant B$, which means the sensitivity of information sent from the user is at most $2B$, thus $(\varepsilon, \delta)$-LDP property of Algorithm \ref{algorithm:reduction bandit} follows directly from the Gaussian mechanism. 

\end{proof}

\begin{proof}[Proof of Theorem~\ref{theorem: BCO utility}]
Note all noise are independently sampled, hence we can fix $Z_1, \dots Z_T$ in advance. Define pseudo loss $\tilde{f}_t(x) = f_t(x) + Z_t$. According to the tail bound of Gaussian variable, there is 
\begin{align}
    \Pr\left[|Z_t| > \sigma \sqrt{2\ln 2T^2} \right] \leqslant \frac{1}{T^2}
\end{align}
By union bound, we have
\begin{align}
    \Pr\left[\exists t \in [T], |Z_t| > \sigma \sqrt{2\ln 2T^2} \right] \leqslant \frac{1}{T}
\end{align}
Define the event $F := \{\exists t \in [T]: |Z_t| > \sigma \sqrt{2\ln 2T^2}\}$, then there is $\Pr[F] \leqslant \frac{1}{T}$.

Once fixed $Z_1, \dots, Z_T$, the output of running Algorithm \ref{algorithm:reduction bandit} over loss sequence $\{f_t| t\in [T]\}$ is the same as the output of running non-private algorithm $\mc{A}$ over pseudo loss sequence $\{\tilde{f_t}| t\in [T]\}$. 

On one hand, we have
\begin{align}
    \mb{E}\left[\sum_{t=1}^T \tilde{f}_t(x_t) - \tilde{f}_t(x)\right] \leqslant & \mb{E}\left[\sum_{t=1}^T \tilde{f}_t(x_t) - \tilde{f}_t(x) | \bar{F}\right] + \Pr[F] \times \mb{E}\left[\sum_{t=1}^T \tilde{f}_t(x_t) - \tilde{f}_t(x) | F \right] \\
    \leqslant & \mb{E}\left[\sum_{t=1}^T \tilde{f}_t(x_t) - \tilde{f}_t(x) | \bar{F}\right] + 2B \\
    \leqslant & (B+\sigma \sqrt{2\ln (2T^2)})\cdot \Reg^T_{\mc{A}} + 2B
\end{align}
On the other hand, according to our definition of $\tilde{f}_t(x)$, there is always
\begin{align}
    \sum_{t=1}^T \tilde{f}_t(x_t) - \tilde{f}_t(x) = \sum_{t=1}^T f_t(x_t) - f_t(x)
\end{align}
Combine above equations, we obtain the conclusion.

For the high probability version, suppose black-box algorithm $\mc{A}$ guarantees that: for any loss sequence $\{\tilde{f}_t(x)\}$ with loss range $\tilde{B}$, with probability at least $1-\kappa$ (over the internal randomness of $\mc{A}$), there is 
\begin{align}
    \forall x \in \mc{X}, \sum_t \tilde{f}_t(x_t) - \sum_t \tilde{f}_t(x) \leqslant \tilde{B} \cdot \mr{Reg}^T_{\mc{A}}
\end{align}
According to union bound and above discussion, we know: $\tilde{B}=B+\sigma \sqrt{2\ln (2T^2)}$, and with probability at least $1-\kappa - \frac{1}{T}$, there is 
\begin{align}
    \forall x \in \mc{X}, \sum_t f_t(x_t) - \sum_t f_t(x) \leqslant \tilde{\mc{O}}\left(\frac{B\ln (T/\delta)}{\varepsilon}\cdot\Reg^T_{\mc{A}}\right)
\end{align}
\end{proof}

\begin{proof}[Proof of Corollary~\ref{cor: concrete BCO utility}]
The guarantee for (strongly) convex and smooth bandit optimization is straightforward by plugging corresponding non-private guarantees in \citet{saha2011improved,hazan2014bandit}. For Stochastic BCO, since our algorithm is equivalent to the case of running any stochastic BCO algorithm over new noise distribution $\mc{Q} \bigotimes \mc{N}(0, \sigma^2)$, where $\bigotimes$ represents the convolution between two distributions, we can use the guarantee for stochastic BCO in \citet{agarwal2011stochastic}.  
\end{proof}

\begin{proof}[Proof of Theorem~\ref{theorem: private MAB both worlds}]
In adversarial setting, using Theorem \ref{theorem: BCO utility} obtains the regret bound. Now we prove the regret bound in stochastic setting. Note for any $k \in [d]$, as the support of original distribution $\mc{V}_k$ is over $[0,1]$, it is a sub-Gaussian distribution with variance proxy $\frac{1}{4}$. Define pseudo distribution $\tilde{\mc{V}}_k = \mc{V}_k \bigotimes \mc{N}(0, \sigma^2)$, where $\bigotimes$ represents the convolution between two distributions. Obviously, the output of Algorithm \ref{algorithm:reduction bandit} over distributions $\{\mc{V}_k | k \in [K]\}$ is the same as the output of non-private algorithm $\mc{A}$ over distributions $\{\tilde{\mc{V}}_k | k \in [K]\}$. As $\tilde{\mc{V}}_k$ is now a sub-Gaussian with variance proxy $\frac{1}{4} + \sigma^2$, hence it's not hard to obtain the conclusion according to the guarantee of $\mc{A}$. 
\end{proof}

\begin{proof}[Proof of Theorem~\ref{BAI utility}]
Just use Theorem 2 in the paper \citet{jamieson2014best} with new sub-Gaussian parameter $\frac{1}{4} + \sigma^2$ 
\end{proof}

\begin{proof}[Proof of Theorem~\ref{two point privacy}]
Since $\abs{f_t(x_{t,1}) - f_t(x_{t,2})} \leqslant 2\rho G \norm{u_t}_2=2\rho G$ and $n_t^\top(x_{t,1}-x_{t,2}) = 2\rho n_t^\top u_t$ which obeys $\mc{N}(0, 4\rho^2\sigma^2)$, the privacy guarantee then follows according to Gaussian mechanism.
\end{proof}

\begin{proof}[Proof of Theorem~\ref{theorem: two point utility}]
Note all noise vectors are independently sampled, hence we can fix $n_1, \dots, n_T$ in advance. Define pseudo loss $\tilde{f}_t(x) = f_t(x) + n^\top_t x$. For any $\{u_t | t \in [T]\}$ in the unit sphere, according to the tail bound of Gaussian variable, there is 
\begin{align}
    \Pr\left[|n^\top_t u_t| > \sigma \sqrt{2\ln 2T^2} \right] \leqslant \frac{1}{T^2}
\end{align}
By union bound, we have
\begin{align}
    \Pr\left[\exists t \in [T], |n^\top_t u_t| > \sigma \sqrt{2\ln 2T^2} \right] \leqslant \frac{1}{T}
\end{align}
Define the event $F := \{\exists t \in [T]: |n^\top_t u_t| > \sigma \sqrt{2\ln 2T^2}\}$, then there is $\Pr[F] \leqslant \frac{1}{T}$.

Once fixed $n_1, \dots, n_T$, the output of running Algorithm \ref{algorithm:reduction two point bandit} over loss sequence $\{f_t| t\in [T]\}$ is the same as the output of running non-private Algorithm \ref{algorithm:non-private two point bandit} over pseudo loss sequence $\{\tilde{f_t}| t\in [T]\}$. 

On one hand, we have
\begin{align}
    & \mb{E}\left[\frac{1}{2}\sum_{t=1}^T \left(\tilde{f}_t(x_{t,1})+\tilde{f}_t(x_{t,2})\right) - \tilde{f}_t(x)\right] \\
    \leqslant & \mb{E}\left[\frac{1}{2}\sum_{t=1}^T \left(\tilde{f}_t(x_{t,1})+\tilde{f}_t(x_{t,2})\right) - \tilde{f}_t(x) | \bar{F}\right] + \Pr[F] \times \mb{E}\left[\frac{1}{2}\sum_{t=1}^T \left(\tilde{f}_t(x_{t,1})+\tilde{f}_t(x_{t,2})\right) - \tilde{f}_t(x) | F \right] \\
    \leqslant & \mb{E}\left[\frac{1}{2}\sum_{t=1}^T \left(\tilde{f}_t(x_{t,1})+\tilde{f}_t(x_{t,2})\right) - \tilde{f}_t(x) | \bar{F}\right] + 2B \\
    \leqslant & \Reg(\mc{A}, G+\sigma\sqrt{d}) + 2B
\end{align}
where $\Reg(\mc{A}, G+\sigma\sqrt{d})$ represents the regret bound of non-private Algorithm \ref{algorithm:non-private two point bandit} for loss functions with Lipschitz constant $G+\sigma\sqrt{d}$.
On the other hand, there is 
\begin{align}
    \mb{E}\left[\sum_{t=1}^T \tilde{f}_t(x_t) - \tilde{f}_t(x)\right] = \mb{E}\left[\sum_{t=1}^T f_t(x_t) - f_t(x)\right]
\end{align}
Combine above equations with the guarantee of non-private Algorithm \ref{algorithm:non-private two point bandit} in \citet{agarwal2010optimal}, we obtain the conclusion.{}
\end{proof}

\subsection{Omitted Proofs in Section 4}
\label{appendix:proofs in section 4}
\begin{proof}[Proof of Theorem \ref{theorem: linear bandit privacy}]
    Since $\norm{x_t} \leqslant 1, y_t \in [-2,2]$ according to our assumption, the privacy guarantee then follows directly from the Gaussian Mechanism, as both the matrix and vector sent to the server satisfy $(\varepsilon/3, \delta/2)$-LDP and $(2\varepsilon/3, \delta/2)$-LDP respectively.
\end{proof}

\begin{proof}[Proof of Theorem \ref{theorem: linear bandit utility}]
    Note our private matrix $\tilde{V}_t$ is an unbiased estimation of true matrix $\sum_{s=1}^t x_s x_s^\top$ with noise $H_t := \sum_{s=1}^t B_s$, where its upper triangular entry obeys the distribution $\mc{N}(0, t\sigma^2)$. Similarly, $\tilde{u}_t$ is an unbiased estimation of true vector $\sum_{s=1}^t y_s x_s$ with noise $h_t := \sum_{s=1}^t \xi_s$, where $h_t \sim \mc{N}(0_d, t\sigma^2 \mr{I}_{d\times d})$. According to the concentration inequality \cite{vershynin2018high}, we know $\norm{H_t}_2 \leqslant \sigma \sqrt{t}(4\sqrt{d} + 2\ln(2T/\alpha)) = \Upsilon_t$ with probability at least $1-\alpha/2T$, thus all the eigenvalues of $H_t + c_t \mr{I}_{d\times d}$ are in the range $[\Upsilon_t, 3\Upsilon_t]$ with high probability. Besides, we have $\norm{h_t}_{(H_t + c_t \mr{I}_{d\times d})^{-1}} \leqslant \sqrt{\Upsilon_t^{-1}}\norm{h_t}_2$, and $\norm{h_t}_2 \leqslant \sigma \sqrt{dt}$ with high probability. Now, using Proposition 4, Proposition 11 and Theorem 5 in paper \cite{shariff2018differentially} with our noise, we obtain the conclusion.  
\end{proof}

\begin{proof}[Proof of Theorem \ref{theorem: gl bandit privacy}]
    Since $\norm{x_t} \leqslant 1, \abs{z_t} \leqslant 1$, and loss function $\ell_t$ is $C$-Lipschitz, the privacy guarantee follows directly from the Gaussian Mechanism, as the matrix, vector, and gradient of any user sent to the server satisfy $(\varepsilon/3, \delta/3)$-LDP respectively.
\end{proof}

\begin{proof}[Proof of Theorem \ref{theorem: gl bandit utility}]
    Define instantaneous regret $r_t:g(x_{t,*}^\top\theta^*)-g(x_t^\top\theta^*)$, then there is $r_t \leqslant G(x_{t,*}^\top\theta^* - x_t^\top\theta^*)$. Besides
    \begin{align*}
        x_t^\top\theta^* + 2\beta_{t-1}\norm{x_t}_{\tilde{V}^{-1}_{t-1}} & \geqslant  x_t^\top\theta^* + \norm{\tilde{\theta}_{t-1}-\theta^*}_{\tilde{V}_{t-1}}\norm{x_t}_{\tilde{V}^{-1}_{t-1}} + \beta_{t-1}\norm{x_t}_{\tilde{V}^{-1}_{t-1}} \\
        & \geqslant x_t^\top\tilde{\theta}_{t-1} + \beta_{t-1}\norm{x_t}_{\tilde{V}^{-1}_{t-1}} \\
        & \geqslant x_{t,*}^\top\tilde{\theta}_{t-1} + \beta_{t-1}\norm{x_{t,*}}_{\tilde{V}^{-1}_{t-1}} \\
        & \geqslant x_{t,*}^\top \theta^*
    \end{align*}
    where the second and the forth inequality is because of our Confidence Ellipsoid Lemma \ref{lemma: confidence ellipsoid}. Thus we have $r_t \leqslant 2G\beta_{t-1}\norm{x_t}_{\tilde{V}^{-1}_{t-1}}$.
    
    Next using common technique in contextual bandits to bound $\sum_t \norm{x_t}_{\tilde{V}^{-1}_{t-1}}$ \cite{shariff2018differentially,jun2017scalable}, we have $\sum_t r_t \leqslant G\beta_T \sqrt{dT\log T}$, which finishes the proof.
\end{proof}

\begin{lemma}[Regret of LDP-OGD]
    \label{lem: regret of LDP-OGD}
    For any convex loss sequence $\{\ell_t(\theta) | t\in [T]\}$ with Lipschitz constant $C$, and for $\forall \theta \in \Theta$, with probability at least $1 - \alpha_1$, we have the following bound 
    \begin{align}
        \quad \sum_{t=1}^T \ell_t(\hat{\theta}_t) - \ell_t(\theta) \leqslant \mc{O}\left(C\sigma\sqrt{dT\ln \frac{T}{\alpha_1}}\right)
    \end{align}
    where $\{\hat{\theta}_t | t \in [T]\}$ are outputs of noisy OGD like step 13 in Algorithm \ref{algorithm: GLM bandit}, and the randomness is over noise $\{r_t | t \in [T]\}$.
\end{lemma}
\begin{proof}
    Condition on the event $\mc{E} = \{\forall t \in [T], \norm{r_t}_2 \leqslant \sqrt{d}C\sigma \}$ (which happens with high probability) and according to the guarantee of On-line Gradient Descent \cite{hazan2016introduction}, there is $\sum_t \ell_t(\hat{\theta}_t) + r_t^\top \hat{\theta}_t - (\ell_t(\theta) + r_t^\top \theta) \leqslant \mc{O}(C\sigma \sqrt{dT})$. Next, using martingale concentration, we know $\norm{\sum_t r_t^\top \hat{\theta}_t}_2 \leqslant \sigma\sqrt{dT}$ and $\norm{\sum_t r_t^\top \theta}_2 \leqslant C\sigma \sqrt{dT}$ with high probability. Combining above three inequalities, we obtain the conclusion.
\end{proof}

\begin{lemma}[Confidence Ellipsoid]
    \label{lemma: confidence ellipsoid}
    In terms of Algorithm \ref{algorithm: GLM bandit}, with probability at least $1 - \alpha_2$, we have the following bound 
    \begin{align}
        \forall t, \quad \norm{\tilde{\theta}_t - \theta^*}_{\tilde{V}_t}^2 \leqslant \tilde{\mc{O}}\left(\frac{C\sigma}{\mu}\sqrt{dT\ln \frac{T}{\alpha_2}}\right)
    \end{align}
    where $\{\tilde{\theta}_t, \tilde{V}_t | t \in [T]\}$ are outputs of Algorithm \ref{algorithm: GLM bandit}, and the randomness is over the injected noise as well as underlying environment.
\end{lemma}
\begin{proof}
    Since $\inf_{a\in (-1,1)} g'(a) = \mu > 0$, it implies loss function $\ell(a, b)$ is $\mu$-strongly convex in terms of the first argument, thus
    \begin{align}
        \sum_{s=1}^t \ell_s(\hat{\theta}_s) - \ell_s(\theta^*) & = \sum_{s=1}^t \ell(x_s^\top \hat{\theta}_s, y_s) - \ell(x_s^\top \theta^*, y_s) \\
        & \geqslant \sum_{s=1}^t \ell'(x_s^\top\theta^*, y_s)(x_s^\top\hat{\theta}_s - x_s^\top\theta^*) + \frac{\mu}{2} (x_s^\top\hat{\theta}_s - x_s^\top\theta^*)^2 \\
        & = \sum_{s=1}^t (-y_s + g(x_s^\top\theta^*)) (x_s^\top\hat{\theta}_s - x_s^\top\theta^*) + \frac{\mu}{2} (x_s^\top\hat{\theta}_s - x_s^\top\theta^*)^2 \\
        & = \sum_{s=1}^t -\eta_s (x_s^\top\hat{\theta}_s - x_s^\top\theta^*) + \frac{\mu}{2} (x_s^\top\hat{\theta}_s - x_s^\top\theta^*)^2
    \end{align}
    Then according to Lemma \ref{lem: regret of LDP-OGD} above, with probability at least $1-\alpha_1$, there is 
    \begin{align}
        \frac{\mu}{2} \sum_{s=1}^t (x_s^\top\hat{\theta}_s - x_s^\top\theta^*)^2 \leqslant \mc{O}\left(C\sigma\sqrt{dt\ln \frac{T}{\alpha_1}}\right) + \sum_{s=1}^t \eta_s (x_s^\top\hat{\theta}_s - x_s^\top\theta^*)
    \end{align}
    Using Corollary 8 in paper \cite{abbasi2012online}, with probability at least $1-\alpha_3$ (over the randomness of noise $\{\eta_t\}$), for all $t$, there is 
    \begin{align}
        \sum_{s=1}^t \eta_s (x_s^\top\hat{\theta}_s - x_s^\top\theta^*) \leqslant \sqrt{\left(2+2\sum_{s=1}^t (x_s^\top(\hat{\theta}_s - \theta^*))^2\right) \cdot \ln \left(\frac{1}{\alpha_3}\sqrt{1+\sum_{s=1}^t (x_s^\top(\hat{\theta}_s - \theta^*))^2}\right)  }
    \end{align}
    Combine above two inequalities, and solve the right hand side using Lemma 2 in paper \cite{jun2017scalable}, then with probability $1-\alpha_1 - \alpha_3$, we have 
    \begin{align}
        \forall t, \quad \sum_{s=1}^t (x_s^\top(\hat{\theta}_s - \theta^*))^2 \leqslant \tilde{\mc{O}}\left(\frac{C\sigma}{\mu}\sqrt{dt\ln \frac{T}{\alpha_1}\ln \frac{T}{\alpha_3}}\right)
    \end{align}
    
    Denote $X_t \in \mb{R}^{t\times d}$ as the design matrix consisting of $x_1,\dots, x_t$, $Z_t = [z_1;z_2;\dots;z_t] \in \mb{R}^t, \bar{B}_t = \sum_{s=1}^t B_s, \bar{\xi}_t = \sum_{s=1}^t \xi_s$. Note 
    \begin{align}
        & \sum_{s=1}^t (x_s^\top(\hat{\theta}_s - \theta^*))^2 \\
        = & \norm{\theta^*}^2_{X_t^\top X_t} - 2Z_t^\top X_t \theta^* + \norm{Z_t}_2^2 \\
        = & \norm{\theta^*}^2_{\tilde{V}_t} - 2\tilde{u}_t^\top \theta^* + \norm{Z_t}_2^2 - \norm{\theta^*}^2_{\bar{B}_t + \tilde{V}_0} + 2\bar{\xi}_t^\top \theta^* \\
        = & \norm{\theta^* - \tilde{\theta}_t}^2_{\tilde{V}_t} - \norm{\tilde{\theta}_t}^2_{\tilde{V}_t} + \norm{Z_t}_2^2 - \norm{\theta^*}^2_{\bar{B}_t + \tilde{V}_0} + 2\bar{\xi}_t^\top \theta^* \\
        = & \norm{\theta^* - \tilde{\theta}_t}^2_{\tilde{V}_t} + \norm{X_t \tilde{\theta}_t - Z_t}_2^2 - \norm{\tilde{\theta}_t}^2_{X^\top_tX_t} + 2\tilde{\theta}_t^\top X_t^\top Z_t - \norm{\tilde{\theta}_t}^2_{\tilde{V}_t} - \norm{\theta^*}^2_{\bar{B}_t + \tilde{V}_0} + 2\bar{\xi}_t^\top \theta^* \\
        = & \norm{\theta^* - \tilde{\theta}_t}^2_{\tilde{V}_t} + \norm{X_t \tilde{\theta}_t - Z_t}_2^2 - \norm{\tilde{\theta}_t}^2_{X^\top_tX_t} + 2\tilde{\theta}_t^\top \tilde{u}_t - \norm{\tilde{\theta}_t}^2_{\tilde{V}_t} - \norm{\theta^*}^2_{\bar{B}_t + c_t \mr{I}} + 2\bar{\xi}_t^\top (\theta^*-\tilde{\theta}_t) \\
        = & \norm{\theta^* - \tilde{\theta}_t}^2_{\tilde{V}_t} + \norm{X_t \tilde{\theta}_t - Z_t}_2^2 - \norm{\tilde{\theta}_t}^2_{X^\top_tX_t} + 2\norm{\tilde{\theta}_t}_{\tilde{V}_t}^2 - \norm{\tilde{\theta}_t}^2_{\tilde{V}_t} - \norm{\theta^*}^2_{\bar{B}_t + c_t \mr{I}} + 2\bar{\xi}_t^\top (\theta^*-\tilde{\theta}_t) \\
        = & \norm{\theta^* - \tilde{\theta}_t}^2_{\tilde{V}_t} + \norm{X_t \tilde{\theta}_t - Z_t}_2^2 +\norm{\tilde{\theta}_t}^2_{\bar{B}_t + c_t \mr{I}} - \norm{\theta^*}^2_{\bar{B}_t + c_t \mr{I}} + 2\bar{\xi}_t^\top (\theta^*-\tilde{\theta}_t) \\
        \geqslant & \norm{\theta^* - \tilde{\theta}_t}^2_{\tilde{V}_t} - \norm{\theta^*}^2_{\bar{B}_t + c_t \mr{I}} - 2 \norm{\bar{\xi}_t}_2 - 2\bar{\xi}_t^\top\tilde{\theta}_t
    \end{align}
    Combine above inequalities, there is 
    \begin{align}
        \norm{\theta^* - \tilde{\theta}_t}^2_{\tilde{V}_t} \leqslant \tilde{\mc{O}}\left(\frac{C\sigma}{\mu}\sqrt{dt\ln \frac{T}{\alpha_1}\ln \frac{T}{\alpha_3}}\right) + \norm{\theta^*}^2_{\bar{B}_t + c_t \mr{I}} + 2 \norm{\bar{\xi}_t}_2 + 2\bar{\xi}_t^\top\tilde{\theta}_t
    \end{align}
    On the other hand, with probability at least $1-\alpha_4$, there is 
    \begin{align}
        \norm{\theta^*}^2_{\bar{B}_t + c_t \mr{I}} \leqslant \tilde{\mc{O}}(\Upsilon_t)=\tilde{\mc{O}}(\sigma\sqrt{dt}) \\
        \norm{\bar{\xi}_t}_2 \leqslant \tilde{\mc{O}}(\sigma\sqrt{dt})
    \end{align}
    and 
    \begin{align}
        \bar{\xi}_t^\top\tilde{\theta}_t & \leqslant \bar{\xi}_t^\top \tilde{V}^{-1}_t(X_t^\top Z_t + \bar{\xi}_t) \\
        & \leqslant \bar{\xi}_t^\top \tilde{V}^{-1}_t X_t^\top Z_t + \tilde{\mc{O}}(\sigma\sqrt{dt}) \\
        & \leqslant \tilde{\mc{O}}(\sigma\sqrt{dt}) 
    \end{align}
    where the last inequality is because $\tilde{V}^{-1}_t X_t^\top Z_t$ is the solution of regularized least square loss function $J(\theta) := \norm{X_t\theta - Z_t}_2^2 + \norm{\theta}^2_{c_t\mr{I}+\bar{B}_t}$. Since $J(\theta^*)\leqslant \tilde{\mc{O}}(\sigma\sqrt{dt})$, and $\Upsilon_t \mr{I} \leqslant c_t\mr{I}+\bar{B}_t \leqslant 3\Upsilon_t \mr{I}$ holds with high probability, there is $\norm{\tilde{V}^{-1}_t X_t^\top Z_t} \leqslant \tilde{O}(1)$, otherwise it cannot be the solution of $J(\theta)$.
    
    Thus, with probability at least $1-\alpha_1-\alpha_3$, we have 
    \begin{align}
        \norm{\tilde{\theta}_t - \theta^*}_{\tilde{V}_t}^2 \leqslant \mc{O}\left(\frac{C\sigma}{\mu}\sqrt{dt\ln \frac{T}{\alpha_1}\ln \frac{T}{\alpha_3}}\right)
    \end{align}
    Taking a union bound over all $T$ rounds and choose appropriate $\alpha_1,\alpha_3$ we then finish the proof.
\end{proof}

\subsection{Discussion about Lower Bound in LDP Contextual Bandits}
\label{appendix: lower bound}
Either for contextual linear bandits or more complex generalized linear bandits, both of our algorithms with LDP guarantee can only achieve $\tilde{\mc{O}}(T^{3/4})$ regret, contrasted with optimal $\mc{O}(T^{1/2})$ regret in non-private case \cite{li2017provably}, as well nearly optimal $\tilde{\mc{O}}(T^{1/2})$ regret for MAB with LDP guarantee. The critical difference is that we need to protect more information in contextual bandits. If we regard MAB as a special case of contextual bandits, decision set $\mc{X}_t$ then becomes $\{e_i | i \in [d]\}$. Privacy of contexts means we need to protect $(e_{I_t}, r_t)$ sent from user $t$ to the server at round $t$, where $I_t$ is the chosen arm and $r_t$ is the reward of user $t$. Recall in Section \ref{subsec:one point feedback}, we only protect $r_t$. Denote $\theta_t$ as the estimation of underlying $\theta^*$ at round $t$, and define $M_t:= \sum_{\tau=1}^t e_{I_\tau} e_{I_\tau}^\top$. Roughly speaking, in almost all analysis of stochastic MAB, the regret bound depends on $\tilde{\mc{O}}(\sqrt{T}\norm{\theta_T - \theta^*}_{M_T})$, and $\norm{\theta_T - \theta^*}_{M_T}$ is nearly a constant in either non-private setting or our MAB example in Appendix \ref{appendix: private MAB}. However in the setting of this section, on one hand, for those sub-optimal arms $i$, the algorithm won't play it too much, and its estimation error $|\theta_T(i)-\theta^*(i)|$ is roughly in constant order. On the other hand, since we still need to protect $e_{I_t}$ at each round, which will lead to an estimation error of $M_T$ in order $\sqrt{T}$. Therefore $\norm{\theta_T - \theta^*}_{M_T}$ is roughly in order $\tilde{\mc{O}}(T^{1/4})$ under LDP setting, which leads to the final $\tilde{\mc{O}}(T^{3/4})$ regret. Though this special case looks a little strange, it shows an inherent difficulty in contextual bandits learning with LDP guarantee, and we conjecture that $\Omega(T^{3/4})$ is exactly the lower bound in this case.

\end{document}